\documentclass{article}\usepackage[]{graphicx}\usepackage[]{color}
\makeatletter
\def\maxwidth{ %
  \ifdim\Gin@nat@width>\linewidth
    \linewidth
  \else
    \Gin@nat@width
  \fi
}
\makeatother

\definecolor{fgcolor}{rgb}{0.345, 0.345, 0.345}

\usepackage{framed}
\makeatletter
 {\par\unskip\endMakeFramed%
 \at@end@of@kframe}
\makeatother

\definecolor{shadecolor}{rgb}{.97, .97, .97}
\definecolor{messagecolor}{rgb}{0, 0, 0}
\definecolor{warningcolor}{rgb}{1, 0, 1}
\definecolor{errorcolor}{rgb}{1, 0, 0}
\newenvironment{knitrout}{}{} 

\usepackage{alltt} 
\usepackage{nips14submit_e,times}
\usepackage{hyperref}
\usepackage{url}

\usepackage{amsmath}
\usepackage{amsthm}
\usepackage{appendix}

\allowdisplaybreaks

\usepackage{amssymb}
\usepackage{bbm}
\usepackage[sort&compress, numbers]{natbib}

\newcommand{\app}[1]{Appendix~\ref{app:#1}}
\newcommand{\lem}[1]{Lemma~\ref{lem:#1}}
\newcommand{\prop}[1]{Proposition~\ref{prop:#1}}
\newcommand{\mysec}[1]{Section~\ref{sec:#1}}

\newcommand{\npp}{\tilde{\eta}} 
\newcommand{\npq}{\eta} 
\newcommand{\mpq}{m} 
\newcommand{\gauss}{\mathcal{N}}

\theoremstyle{plain}
\newtheorem{theorem}{Theorem}[section]
\newtheorem{proposition}[theorem]{Proposition}
\newtheorem{lemma}[theorem]{Lemma}

\title{Covariance Matrices for Mean Field Variational Bayes}

\author{
Ryan Giordano\\
Department of Statistics\\
University of California, Berkeley\\
\texttt{rgiordano@berkeley.edu} \\
\And
Tamara Broderick \\
Department of Statistics\\
University of California, Berkeley\\
\texttt{tab@stat.berkeley.edu} \\
}

\nipsfinalcopy 

\newcommand{\eq}[1]{Eq.~(\ref{eq:#1})}
\newcommand{\eqw}[1]{Eq.~(#1)}
\newcommand{\fig}[1]{Fig.~(\ref{fig:#1})}
\newcommand{\kl}{\textrm{KL}}
\DeclareMathOperator*{\argmin}{arg\,min}
\newcommand{\mbe}{\mathbb{E}}
\newcommand{\var}{\textrm{Var}}
\newcommand{\cov}{\textrm{Cov}}
\IfFileExists{upquote.sty}{\usepackage{upquote}}{}
\begin{document}

\maketitle


\section{Introduction}

With increasingly efficient data collection methods, scientists are
interested in quickly analyzing ever larger data sets. 
In particular, the promise of these large data sets is not simply to fit 
old models but instead to learn more nuanced patterns from data
than has been possible in the past.
In theory, the Bayesian paradigm promises exactly these desiderata.
Hierarchical modeling allows practitioners to capture complex relationships
between variables of interest. Moreover, Bayesian analysis allows practitioners
to quantify the uncertainty in any model estimates---and to do so coherently
across all of the model variables.

\emph{Mean Field Variational Bayes} (MFVB), a method for approximating
a Bayesian posterior distribution, has grown in 
popularity due to its fast runtime on large-scale data sets
\cite{blei:2003:lda, blei:2006:dp, hoffman:2013:stochastic}.
But it is well known
that a major failing of MFVB is that it gives underestimates
of the uncertainty of model variables that can be almost
arbitrarily worse and provides no information about how
the uncertainties in different model variables interact
\cite{wang:2005:inadequacy, bishop:2006:pattern, rue:2009:approximate, turner:2011:two}.
We develop a fast, general methodology for exponential families that augments MFVB
to deliver accurate uncertainty estimates for model variables---both for individual
variables and coherently across variables.
In particular, as we elaborate in Section~\ref{sec:mfvb},
MFVB for exponential families defines a fixed-point equation in the means
of the approximating posterior, and our approach yields a covariance estimate
by perturbing this fixed point.
Inspired by 
\emph{linear response theory}, which
has previously been applied to Boltzmann machines \cite{kappen:1998efficient}
and loopy belief propagation \cite{welling:2004:linear},
we call our method \emph{linear response variational Bayes} (LRVB).

We demonstrate the accuracy of our 
covariance estimates with experiments on simulated data from a mixture of normals.
Specifically, we show that the LRVB variance estimates are nearly identical to
those produced by a Metropolis-Hastings sampler,
even when MFVB variance is dramatically underestimated.  We also show how the ability
to analytically propagate uncertainty through a graphical model allows the easy computation
of the influence of data points on parameter point estimates, i.e.
``graphical model leverage scores.''
While the data sets we examine in our experiments below (Section~\ref{sec:experiments}) are simulated,
in future work we will demonstrate the applicability and scalability
of LRVB on larger, experimentally-obtained data sets.

\section{Mean-field variational Bayes in exponential families} \label{sec:mfvb}

Denote our $N$ observed data points by the $N$-long column vector $x$,
and denote our unobserved model parameters by $\theta$. Here, $\theta$ is a column
vector residing in some space $\Theta$; it has $J$ subgroups and
total dimension $D$. Our model is specified by a distribution of the
observed data given the model parameters---the likelihood $p(x | \theta)$---and
a prior distributional belief on the model parameters
$p(\theta)$. Bayes' Theorem yields the posterior $p(\theta | x)$.

MFVB approximates $p(\theta | x)$ by a factorized distribution of the
form $q(\theta) = \prod_{j=1}^{J} q(\theta_{j})$ such that the
Kullback-Liebler divergence $\kl(q || p)$ between $q$ and $p$ is
minimized:
$$
  q^{*} := \argmin_{q} \kl(q || p) = \argmin_{q} \mbe_{q} \left[ \log p(\theta | x) - \sum_{j = 1:J} \log q(\theta_{j}) \right].
$$
By the assumed $q$ factorization, the solution to this
minimization obeys the following fixed point equations \cite{bishop:2006:pattern}:
\begin{equation}
	\label{eq:kl_div}
	\log q^{*}_{j}(\theta_{j}) = \mbe_{q^{*}_{i}: i \ne j} \log p(\theta, x) + \mathrm{constant}.
\end{equation}

For index $j$, suppose that
$p(\theta_{j} | \theta_{i: i \ne j}, x)$ is in natural exponential family form:
\begin{equation}
  \label{eq:variational_exp_def}
  p(\theta_{j} | \theta_{i: i \ne j}, x) = \exp(\npp_j^{T} \theta_{j} - A_j(\npp_j))
\end{equation}
with local natural parameter $\npp_j$ and local log partition function $A_j$.
Here, $\npp_j$ may be a function of $\theta_{i: i \ne j}$ and $x$.
If the exponential family assumption above holds for every index $j$,
then we can write
$\npp_{j} = \sum_{R \subseteq [J] \backslash \{j\}} G_{R} \prod_{r \in R} \theta_{r}$,
where $[J] := \{1,\ldots, J\}$ and $G_{R}$ is a
constant in all of $\theta$ (\app{exp_fams}).
It follows from \eq{kl_div} and the assumed factorization of $q^{*}$ that
\begin{equation} \label{eq:exp_approx_marg}
	\log q^{*}_{j}(\theta_{j}) =
            \left\{ \sum_{R \subseteq [J] \backslash \{j\}} G_{R} \prod_{r \in R}
                    [\mbe_{q^{*}_r} \theta_{r}] \right\}^{T} \theta_{j} + \mathrm{constant}
\end{equation}
In particular, we see that $q^{*}_{j}$ is in the same
exponential family form
as $p(\theta_{j} | \theta_{i: i \ne j}, x)$. Let $\npq_{j}$
denote the natural
parameter of $q^{*}_{j}$, and denote the mean parameter of
$q^{*}_{j}$ as $\mpq_{j} :=  \mbe_{q^{*}_{j}} \theta_j$. We see from
\eq{exp_approx_marg} that
$\npq_{j} = \sum_{R \subseteq [J] \backslash \{j\}} G_{R} \prod_{r \in R} \mpq_{r}$.
Since $\mpq_{j}$ is a function of $\npq_{j}$,
we have the fixed point equations
$\mpq_{j} = M_{j}(\mpq_{i:i \ne j})$ for mappings $M_{j}$ across $j$ and $\mpq = M(\mpq)$
for the vector of mappings $M$.

\section{Linear response} \label{sec:lr}

Now define $p_{t}(\theta|x)$ such that its log is a linear
perturbation of the log posterior:
\begin{equation} \label{eq:perturbed_dens}
  \log p_{t}(\theta | x) = \log p(\theta | x) + t^{T} \theta - c(t),
\end{equation}
where $c(t)$ is a constant in $\theta$. Since $c(t)$ normalizes
the $p_{t}(\theta | x)$ distribution, it is in fact the cumulant
generating function of $p(\theta | x)$. Further, every
conditional distribution $p_{t}\left(\theta_{j} | \theta_{i: i \ne j}, x\right)$
is in the same exponential family as every conditional
distribution $p\left(\theta_{j} | \theta_{i: i \ne j}, x\right)$ 
by construction. So, for each $t$, we have mean field 
variational approximation $q_{t}^{*}$ with
marginal means $\mpq_{t,j} := E_{q_{t}^{*}} \theta_j$ and
fixed point equations
$\mpq_{t,j} = M_{t,j}(\mpq_{t,i:i \ne j})$ across $j$; hence $\mpq_t = M_t(\mpq_t)$.
Taking derivatives of the latter relationship with respect to $t$, we find
\begin{equation} \label{eq:mean_derivs}
  \frac{d \mpq_t}{d t^T}
    = \frac{\partial M_t}{\partial \mpq_t^T} \frac{d \mpq_t}{d t^T}
      + \frac{\partial M_t}{\partial t^T}.
\end{equation}
In particular, note that $t$ is a vector of size $D$ (the total dimension
of $\theta$), and $\frac{d \mpq_t}{d t^T}$, e.g., 
is a matrix of size $D \times D$ with $(a,b)$th entry equal to the scalar $d\mpq_{t,a} / dt_{b}$.

Since $q_t^{*}$ is the MFVB approximation for the perturbed posterior $p_t(\theta|x)$,
we may hope that $\mpq_t = E_{q_{t}^{*}} \theta$ is close to the 
perturbed-posterior mean $\mbe_{p_t} \theta$.  The practical success of
MFVB relies on the fact that this approximation is often good in practice.
To derive interpretations of the individual terms in \eq{mean_derivs},
we assume that this equality of means holds,
but we indicate where we use this assumption with
an approximation sign: $\mpq_t \approx \mbe_{p_t} \theta$.
A fuller derivation of the next set of equations is given in
\app{lr_derivs}.
\begin{equation} \label{eq:dm_dt}
  \frac{d \mpq_t}{d t^T} \approx \frac{d}{dt^T} \mbe_{p_t} \theta = \Sigma_{p_{t}}
  \quad \textrm{ and } \quad
  \frac{\partial M_t}{\partial t^T} =
       \frac{\partial}{\partial t^T} \mbe_{q_t^*} \theta = \Sigma_{q_{t}^*}
  \quad \textrm{ and } \quad
  \frac{d M_t}{d \mpq_t^T} = \Sigma_{q_{t}^*} \frac{ \partial \npq_t }{ \partial \mpq_t^T },
\end{equation}
where $\Sigma_{p_{t}}$ is the covariance matrix of $\theta$ under $p_{t}$,
$\Sigma_{q_{t}^*}$ is the covariance matrix of $\theta$ under $q_t^*$,
and $\npq_t = (\npq_{t,1}^T, \ldots, \npq_{t,J}^T)^{T}$ is the vector defined
by stacking natural parameters from each $q_{t,j}^{*}$ distribution.

Now let $H := \left. \frac{ \partial \npq_t }{ \partial \mpq_t^T } \right|_{t=0}$.
Then substituting \eq{dm_dt} into \eq{mean_derivs} and
evaluating at $t=0$, we find
\begin{equation} \label{eq:lrvb_est}
  \Sigma_{p} \approx \Sigma_{q^*} H \Sigma_{p} + \Sigma_{q^*}
    \quad \Rightarrow \quad
    \Sigma_{p} \approx (I - \Sigma_{q^*} H)^{-1} \Sigma_{q^*}
\end{equation}
Thus, we call $\hat{\Sigma} := (I - \Sigma_{q^*} H)^{-1} \Sigma_{q^*}$
the LRVB estimate of the true posterior covariance $\Sigma_{p}$.
\footnote{\eq{lrvb_est} involves the inverse of a matrix as large
as the total number of natural parameters, which in many problems
can be impractical.  However, since the variational covariance
$\Sigma_{q^*}$ is block diagonal and $H$ is often sparse, one may
be able to use Schur complements to efficiently find sub-matrices of $\hat\Sigma$.
In \app{leverage} we work through two examples of this technique.}

\section{Experiments} \label{sec:experiments}

\subsection{Mixture of normals} \label{sec:normal_mix}

Mixture models constitute some of the most popular models for MFVB application
\cite{blei:2003:lda, blei:2006:dp} and are often used as an example
of where MFVB covariance estimates may go awry \cite{bishop:2006:pattern, turner:2011:two}.
Here we focus on a
$K$-component, one-dimensional mixture of normals likelihood.
In what follows, $\pi_k$ is the probability of the $k$th component,
$\gauss$ denotes the univariate normal distribution,
$\mu_k$ is the mean of the $k$th component, and $\tau_k$ is the 
precision of the $k$th component (so $\tau^{-1}_k$ is
variance). $N$ is the number of data points, and $x_{n}$
is the $n$th observed data point. Then the likelihood is
\begin{equation} \label{eq:normal_mixture_model}
  p(x | \pi, \mu, \tau) =
    \prod_{n=1:N} \sum_{k=1:K} \pi_{k} \gauss(x_n | \mu_k, \tau^{-1}_k).
\end{equation}
To complete the generative model, we assign priors
\begin{equation} \label{eq:mix_priors}
  \pi \sim \mathrm{Dirichlet}_{K}(1),
    \quad \tau \sim \mathrm{Gamma}(2.0001, 0.1),
    \quad \mu \sim \gauss(0, 100).
\end{equation}


We wish to approximate the covariance matrix of the
parameters $\log(\pi), \mu, \log(\tau)$ in the posterior
distribution $p(\pi, \mu, \tau | x)$
from the preceding generative model.
In our experiment, $K=3$ and $N=3000$
for each of 100 simulations.
We compare three different approaches to compute the posterior covariance:
a Metropolis-Hastings (MH) sampler, MFVB, and LRVB.  The MH sampler
draws independent proposals centered at the MAP estimate in order to
avoid label-switching problems.  The two VB solutions augment
\eq{normal_mixture_model} with indicator variables, $z$, which indicate which
component each point was drawn from.
We note that for each of the parameters $\log(\pi)$, $\mu$, and $\log(\tau)$,
both MH and MFVB produce point estimates close to the true values,
so our key assumption in the LRVB derivations of Section~\ref{sec:lr}
appears to hold.
To compare the covariance matrices, we use MH as a ground truth;
for the low-dimensional model we are using, it is reasonable
to expect that MH should return a good approximation of the true posterior.
We see in \fig{SimulationStandardDeviations}
that the LRVB estimates agree with the MH posterior variance
while MFVB consistently underestimates the posterior variance.

\begin{knitrout}
\definecolor{shadecolor}{rgb}{0.969, 0.969, 0.969}\color{fgcolor}\begin{figure}[ht!]

{\centering \includegraphics[width=.32\linewidth,height=.4\linewidth]{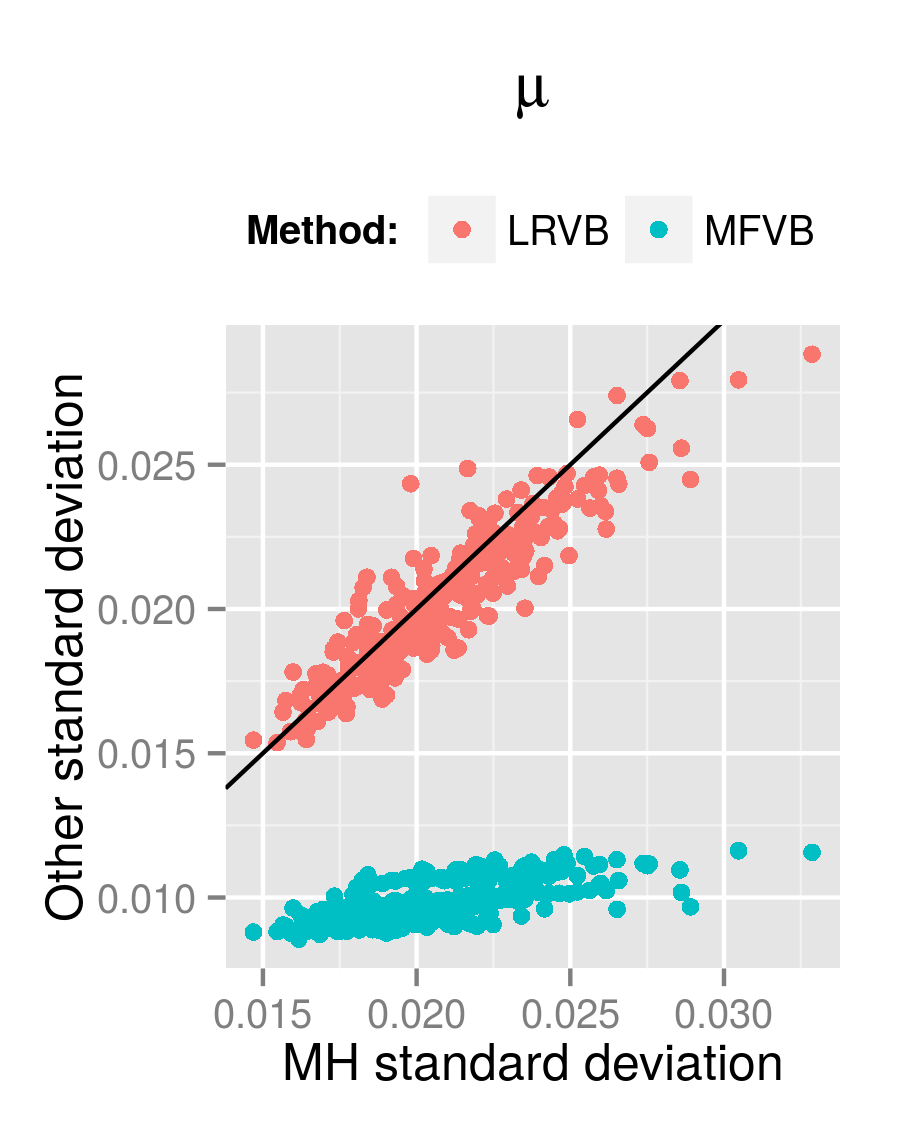} 
\includegraphics[width=.32\linewidth,height=.4\linewidth]{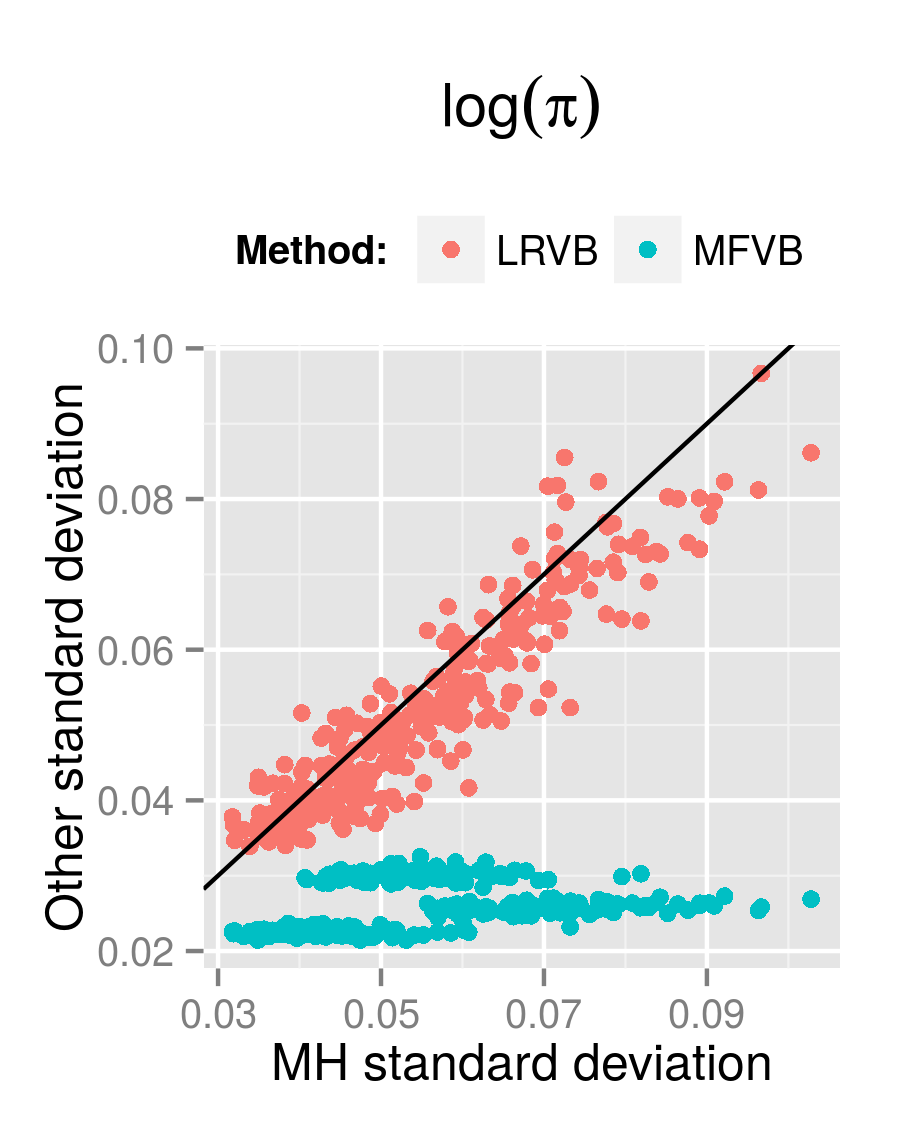} 
\includegraphics[width=.32\linewidth,height=.4\linewidth]{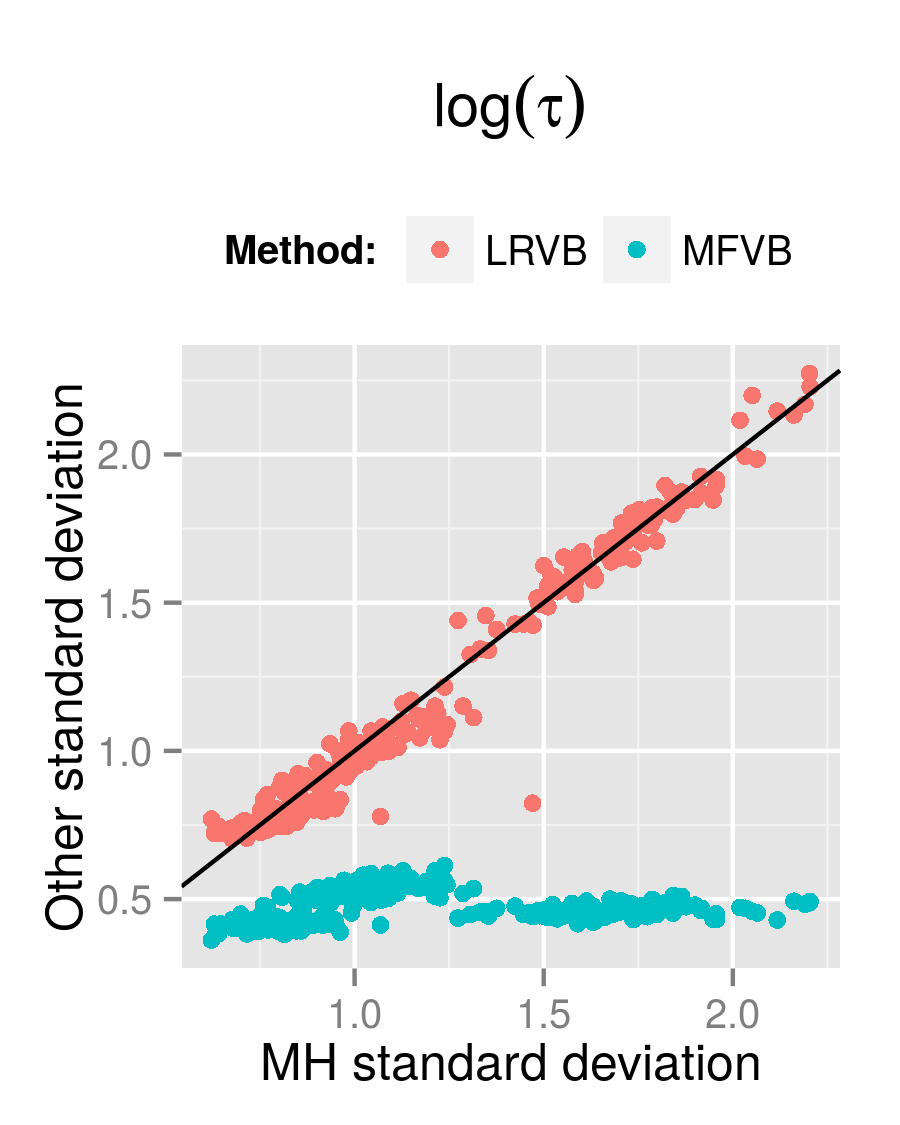} 

}

\caption[Comparison of estimates of the posterior standard deviation for each model parameter according to methods MH, MFVB, and LRVB and across 100 simulations]{Comparison of estimates of the posterior standard deviation for each model parameter according to methods MH, MFVB, and LRVB and across 100 simulations.\label{fig:SimulationStandardDeviations}}
\end{figure}

\end{knitrout}

\subsection{Sensitivity analysis} \label{sec:normal_sensitivity}

Next consider a slight variation to the model of \mysec{normal_mix}.
We retain the distribution of $p(x | \pi, \mu, \tau)$ in \eq{normal_mixture_model}
but now assume that the observed data $x^*$ are actually independent
noisy versions of $x$:
$
  x_n^* \sim \gauss(x_n, \sigma^2),
$
for a deterministic constant $\sigma^2$.
We retain the prior on $\mu$ in \eq{mix_priors}, but fix $\pi$ and $\tau$ at their
true values. In this new model, $x$ and $\mu$
are the unknown parameters.  Using LRVB, we can estimate the posterior covariance
between any $x_n$ and the mixture parameters $\mu$.
If we look at this covariance as $\sigma^2 \rightarrow 0$,
we obtain a type of \emph{leverage score}. That is, the
limiting value of this covariance can be used to estimate the influence
of observation $x_n$ on the mixture parameters in the spirit of classical
linear model leverage scores from the statistics literature.
LRVB leads to a straightforward analytic expression for these
covariances, which can be found in \eq{lrvb_lev_scores} in \app{leverage}.
\footnote{\app{leverage} also includes a proof
that this LRVB-limiting method reproduces classical leverage scores when
applied to linear regression.}

Note that these leverage scores are impossible to compute in naive MFVB,
since they involve correlations between distinct mean field components, and
difficult to compute using MH, since they require estimating a large number
of very small covariances with a finite number of draws.

To evaluate these LRVB-derived leverage scores, we compare them to the
effect of manually perturbing our data and re-fitting the model.
Here, we choose $K=2$ components in the mixture model and $N=500$.  (The
small $N$ is chosen to make the manual perturbation calculations more manageable.)
The LRVB-derived leverage scores
are plotted as a function of $x_n$ location on the lefthand side of \fig{LeverageGraph}.
We can see from the comparison on the righthand side of 
\fig{LeverageGraph} that the LRVB-derived leverage scores
match well with the results of manual perturbation, which
took over 30 times longer to compute.

\begin{knitrout}
\definecolor{shadecolor}{rgb}{0.969, 0.969, 0.969}\color{fgcolor}\begin{figure}[ht!]

{\centering \includegraphics[width=.49\linewidth,height=.3755\linewidth]{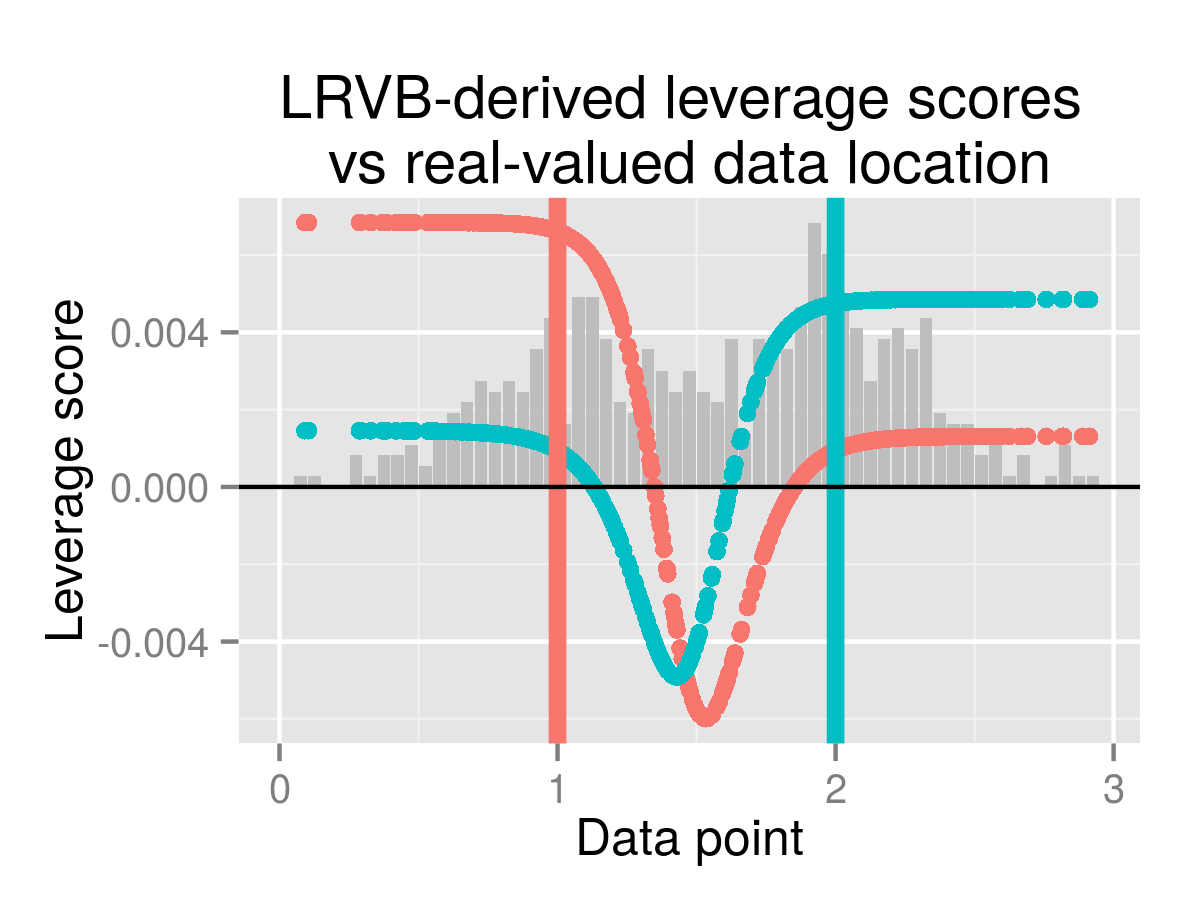} 
\includegraphics[width=.49\linewidth,height=.3755\linewidth]{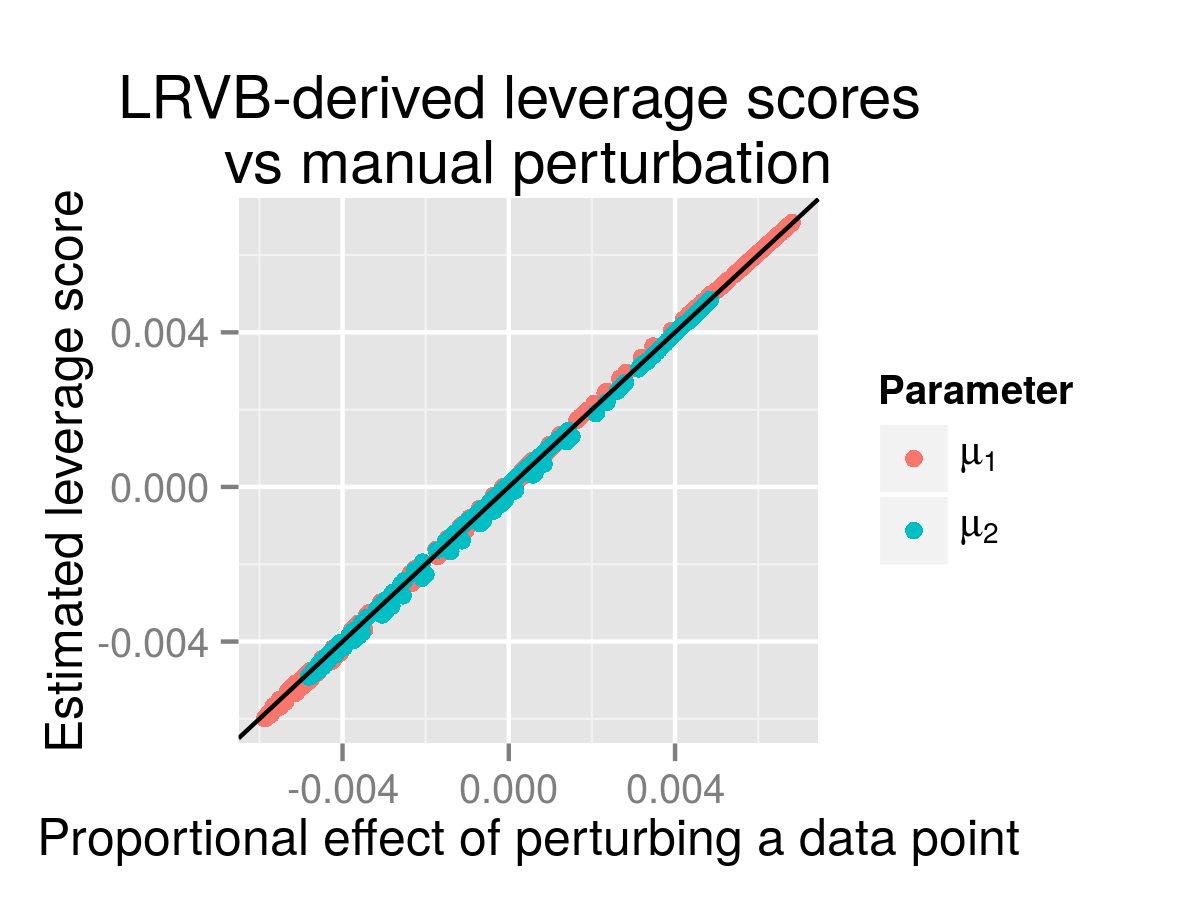} 

}

\caption[Left]{Left: LRVB leverage scores by data point location. Right: leverage comparison.\label{fig:LeverageGraph}}
\end{figure}

\end{knitrout}

As expected, the data points with the greatest effect on the location of a
component are the ones most likely to be assigned to the component.
Interestingly, though, data still retain leverage on a component even when they
are assigned with certainty to the other component.  
Indeed, a data point assigned to one component with probability close to one
will affect that component's mean, which in turn affects the classification of other
data points, which then affects the location of the other component.
In this way, we see that LRVB is
estimating covariances that are the results of complex chains of correlations.

\subsubsection*{Acknowledgments}

The authors thank Michael I.~Jordan for suggesting that we look at linear
response theory and Alex Blocker for helpful comments.
R.~Giordano and T.~Broderick were funded by Berkeley Fellowships.

\bibliographystyle{unsrt}
\bibliography{nips_workshop_abstract}

\newpage
\appendix
\appendixpage

\section{Derivations} \label{app:derivs}

\subsection{MFVB for conditional exponential families} \label{app:exp_fams}

\begin{lemma}
Suppose \eq{variational_exp_def} holds across all $j$; that is,
$$
p(\theta_{j} | \theta_{i: i \ne j}, x) = \exp(\npp_j^{T} \theta_{j} - A_j(\npp_j)).
$$
Then, for the natural parameter
$\npp_{j}$, we have
$$
  \npp_{j} = \sum_{R \subseteq [J] \backslash \{j\}} G_{R} \prod_{r \in R} \theta_{r},
$$
where $[J] := \{1,\ldots,J\}$ and $G_R$ is constant in $\theta$.
\end{lemma}

\begin{proof}
We see that 
$\log p(\theta | x) = \log p(\theta_{j} | \theta_{i: i \ne j}, x) + \log p(\theta_{i: i \ne j} | x)$
depends on $\theta_j$ only via the first term in the sum. So by \eq{variational_exp_def},
$\log p(\theta | x)$ is linear in $\theta_j$, and we can write
$$
  \log p(\theta | x) = \sum_{R' \subset [J]} G'_{R'} \prod_{r \in R'} \theta_r + \mathrm{constant},
$$
where $G'_{R'}$ and the final summand are constant in all of $\theta$.
The result follows by collecting those terms where $j \in R'$.
\end{proof}

\subsection{Linear response} \label{app:lr_derivs}

We here derive the three equalities in \eq{dm_dt}, which appear
as three propositions below. In these propositions,
we assume that $p(\theta | x)$
is in the exponential family as above. We will further assume
that all natural parameters (for $p$ or variational approximations)
are in the interior of the parameter space
and that $t$ is sufficiently small.
These assumptions
will allow us to apply dominated convergence (cf.\ Section 2.3
of \cite{keener:2010:theoretical}).

\begin{proposition} \label{prop:dm_dt}
  $\frac{d}{dt} \mbe_{p_t} \theta = \Sigma_{p_{t}}$.
\end{proposition}

\begin{proof}

\begin{align*} 
  \frac{d}{d t^{T}} \mbe_{p_{t}} \theta
    &= \frac{d}{d t^{T}} \int_{\theta}
		      \theta e^{t^{T} \theta - c(t)}
        p(\theta | x) d\theta
        \quad \textrm{by the definition of $p_t$ in \eq{perturbed_dens}} \\
    &= \int_{\theta} \theta
          \left[
            \frac{d}{d t^{T}} e^{t^{T} \theta - c(t)}
          \right]
        p(\theta | x) d\theta
        \quad \textrm{by dominated convergence} \\
	  &= \int_{\theta} \theta \theta^{T} e^{t^{T} \theta - c(t)} p(\theta | x) d\theta
		  - \int_{\theta} \theta e^{t^{T} \theta - c(t)} p(\theta | x) d\theta
				\cdot \frac{d c(t)}{d t^{T}}
        \\
    &= \mbe_{p_t}\left[\theta \theta^T\right] -
       \mbe_{p_t}\left[\theta\right] \mbe_{p_t}\left[\theta\right]^T = \Sigma_{p_{t}}
\end{align*}

\end{proof}

To approximate $\frac{d \mpq_{t}}{d t^{T}}$, we assume
not only that $\mpq_{t} \approx \mbe_{p_t} \theta$ for any particular $t$ but
further that $\mpq_t$ tracks the true mean $\mbe_{p_t} \theta$ as $t$ varies.
In this case, by \prop{dm_dt}, we have
$$
  \frac{d \mpq_{t}}{d t^{T}}
    \approx \frac{d}{d t^{T}} \mbe_{p_{t}} \theta
    = \Sigma_{p_{t}},
$$
the first (approximate) equality in \eq{dm_dt}.

To derive the final two equalities in \eq{dm_dt}, we make use of the following lemma.

\begin{lemma} \label{lem:dM_deta}
  $M_{t,j}$ depends on $t$ only via $\npq_{t,j}$, the natural parameter of the $q_{t,j}^{*}$ distribution.
  And $\frac{dM_{t,j}}{d\eta_{t,j}^T} = \Sigma_{q_{t,j}^{*}}$.
\end{lemma}

\begin{proof}
The first part of the lemma follows from writing the definition of $M_{t,j}$:
$$
  M_{t,j}
    = \mbe_{q_{t,j}^{*}} \theta_j
    = \int_{\theta_j}
            \theta_j \exp \left(
              \npq_{t,j}^{T} \theta_j - A_j\left(\npq_{t,j}\right)
            \right)
          d \theta_j.
$$

For the second part,
\begin{align*}
  \frac{d M_{t,j}}{d \npq_{t,j}^T}
      &= \int_{\theta_j}
          \frac{d}{d \npq_{t,j}^T}
            \theta_j \exp \left(
              \npq_{t,j}^{T} \theta_j - A_j\left(\npq_{t,j}\right)
            \right)
          d \theta_j \quad \textrm{by dominated convergence} \\
      &= \int_{\theta_j}
            \theta_j
              \left[ \theta_j^{T} - \mbe_{q_{t,j}^{*}} \theta_j^T \right]
              \exp \left(
                \npq_{t,j}^{T} \theta_j - A_j\left(\npq_{t,j}\right)
              \right)
          d \theta_j \\
      &= \Sigma_{q_{t,j}^{*}}
\end{align*}
\end{proof}

\begin{proposition}
$
  \frac{\partial M_t}{\partial t^T} = \Sigma_{q_{t}^*}.
$
\end{proposition}

\begin{proof}
By \lem{dM_deta}, we have for any indices $i$ and $j$ in $[J]$ that
\begin{equation} \label{eq:dM_dt}
  \frac{\partial M_{t,j}}{\partial t_i^T}
    = \frac{d M_{t,j}}{d \npq_{t,j}^T} \frac{\partial \npq_{t,j}}{\partial t_i^T},
\end{equation}
where the first factor is also given by \lem{dM_deta}.
It remains to find the second factor, $\frac{\partial \npq_{t,j}}{\partial t_i^T}$.
By the discussion
after \eq{variational_exp_def} and the construction of $p_t$,
the natural parameter $\npp_{t,j}$
of $p_{t}\left(\theta_{j} | \theta_{i: i \ne j}, x\right)$
satisfies
$$
  \npp_{t,j} = \sum_{R \subseteq [J]/\{j\}} G_{R} \prod_{r \in R} \theta_r + t_j.
$$
So, as in the derivation of \eq{exp_approx_marg},
the natural parameter $\npq_{t,j}$ of $q_j^*(\theta_j)$ 
satisfies
\begin{equation} \label{eq:etat}
  \npq_{t,j}
    = \sum_{R \subseteq [J]/\{j\}} G_{R} \prod_{r \in R} \mpq_{t,r} + t_j
\end{equation}
for $\mpq_{t,r} := \mbe_{q^*_{t,r}} \theta_{r}$.

Let $d_j$ be the dimension of $\theta_j$ and hence the
dimension of $\npq_{t,j}$ and $t_j$. Hence,
$$
  \frac{\partial \npq_{t,j}}{\partial t_i^T}
    = \left\{ \begin{array}{ll}
        I_{d_j} & j = i \\
        0_{d_j,d_i} & \mathrm{else}
      \end{array} \right. ,
$$
where $I_{a}$ is the identity matrix of dimension $a$, and $0_{a,b}$ is the
all zeros matrix of dimension $a \times b$.

Finally, by \eq{dM_dt}, \lem{dM_deta}, and the expression for
$\frac{\partial \npq_{t,j}}{\partial t_i^T}$ just obtained, we have
$$
  \frac{\partial M_t}{\partial t^T} = \Sigma_{q_{t}^*} I_{D} = \Sigma_{q_{t}^*}.
$$
\end{proof}

\begin{proposition}
$
  \frac{d M_t}{d m_t^T} = \Sigma_{q_{t}^*} \frac{ \partial \eta_t }{ \partial m_t^T }.
$
\end{proposition}

\begin{proof}
By \lem{dM_deta} and analogous to \eq{dM_dt}, we have
\begin{equation} \label{eq:dM_dm}
  \frac{\partial M_{t,j}}{\partial m_{t,i}^T}
    = \frac{d M_{t,j}}{d \npq_{t,j}^T} \frac{\partial \npq_{t,j}}{\partial m_{t,i}^T}.
\end{equation}
The result follows immediately from \lem{dM_deta}.
\end{proof}

\section{Multivariate normal} \label{app:SEM}

For any target distribution $p(\theta | x)$, it is well-known that MFVB cannot be used
to estimate the covariances between the components of $\theta$.
In particular, if $q^*$ is the estimate
of $p(\theta | x)$ returned by MFVB,
$q^*$ will have a block-diagonal covariance matrix---no matter the form
of the covariance of $p(\theta | x)$. 
By contrast, the next result shows that the LRVB covariance estimate is exactly correct in the 
case where the target distribution, $p(\theta|x)$, is (multivariate) normal.

In order to prove this result, we will rely on the following lemma.
\begin{lemma} \label{lem:lrvb_mvn}
  Consider a target posterior distribution characterized by $p(\theta | x) = \gauss(\theta | \mu, \Sigma)$,
  where $\mu$ and $\Sigma$ may depend on $x$, and $\Sigma$ is invertible.
  Let $\theta = (\theta_{1}, \ldots, \theta_{J})$,
  and consider a MFVB approximation to $p(\theta| x)$ that factorizes as $q(\theta) = \prod_{j} q(\theta_j)$.
  Then the variational posterior means are the true posterior means; i.e. $m_j = \mu_j$ for all $j$ between 
  $1$ and $J$.
\end{lemma}

\begin{proof}
  The derivation of MFVB for the multivariate normal can be found in Section 10.1.2 of
  \cite{bishop:2006:pattern}; we highlight some key results here.
  Let $\Lambda = \Sigma^{-1}$. Let the $j$ index on a row or column correspond to $\theta_j$,
  and let the $-j$ index
  correspond to $\{\theta_{i}: i \ne j\}$. E.g., for $j=1$,
  $$
    \Lambda
      = \left[ \begin{array}{ll}
          \Lambda_{11} & \Lambda_{1,-1} \\
          \Lambda_{-1,1} & \Lambda_{-1,-1}
        \end{array} \right].
  $$
  By the assumption that $p(\theta | x) = \gauss(\theta | \mu, \Sigma)$, we have
\begin{equation}\label{eq:mvn_variational_dist}
    \log p(\theta_{j} | \theta_{i: i \ne j}, x)
      = -\frac{1}{2} (\theta_{j} - \mu_{j})^{T} \Lambda_{jj} (\theta_j - \mu_j) +
         (\theta_{j} - \mu_{j})^{T} \Lambda_{j,-j} (\theta_{-j} - \mu_{-j}) + \mathrm{constant},
\end{equation}

  where the final term is constant in $\theta_{j}$.
  It follows that
  \begin{align*}
    \log q^{*}_{j}(\theta_j)
      &= \mbe_{q^{*}_{i}: i \ne j} \log p(\theta, x) + \mathrm{constant} \\
      &= -\frac{1}{2} \theta_{j}^{T} \Lambda_{jj} \theta_j + \theta_j \mu_j \Lambda_{jj} - \theta_j \Lambda_{j,-j} (\mbe_{q^{*}} \theta_{-j} - \mu_{-j}).
  \end{align*}
  So 
  \begin{equation*}
    q^*_j(\theta_j) = \gauss(\theta_j | m_{j}, \Lambda_{jj}^{-1}),
  \end{equation*}
  with mean parameters
  \begin{equation} \label{eq:mvn_stable_point}
    m_{j} = \mbe_{q^{*}_j} \theta_j = \mu_{j} - \Lambda_{jj}^{-1} \Lambda_{j,-j} (m_{-j} - \mu_{-j})
  \end{equation}
  as well as an equation for $\mbe_{q^{*}} \theta^T \theta$.

Note that $\Lambda_{jj}$ must be invertible, for if it
were not, $\Sigma$ would not be invertible.

The solution $m = \mu$ is a unique stable point for
\eq{mvn_stable_point}, since the fixed point equations for each $j$
can be stacked and rearranged to give
\begin{eqnarray*}
m-\mu & = & -\left[\begin{array}{ccccc}
0 & \Lambda_{11}^{-1}\Lambda_{12} & \cdots & \Lambda_{11}^{-1}\Lambda_{1\left(J-1\right)} & \Lambda_{11}^{-1}\Lambda_{1J}\\
\vdots &  & \ddots &  & \vdots\\
\Lambda_{JJ}^{-1}\Lambda_{J1} & \Lambda_{JJ}^{-1}\Lambda_{J2} & \cdots & \Lambda_{JJ}^{-1}\Lambda_{J\left(J-1\right)} & 0
\end{array}\right]\left(m-\mu\right)\\
 & = & -\left[\begin{array}{ccccc}
\Lambda_{11}^{-1} & \cdots & 0 & \cdots & 0\\
\vdots & \ddots &  &  & \vdots\\
0 &  & \ddots &  & 0\\
\vdots &  &  & \ddots & \vdots\\
0 & \cdots & 0 & \cdots & \Lambda_{JJ}^{-1}
\end{array}\right]\left[\begin{array}{ccccc}
0 & \Lambda_{12} & \cdots & \Lambda_{1\left(J-1\right)} & \Lambda_{1J}\\
\vdots &  & \ddots &  & \vdots\\
\Lambda_{J1} & \Lambda_{J2} & \cdots & \Lambda_{J\left(J-1\right)} & 0
\end{array}\right]\left(m-\mu\right)\Leftrightarrow\\
0 & = & \left[\begin{array}{ccccc}
\Lambda_{11} & \cdots & 0 & \cdots & 0\\
\vdots & \ddots &  &  & \vdots\\
0 &  & \ddots &  & 0\\
\vdots &  &  & \ddots & \vdots\\
0 & \cdots & 0 & \cdots & \Lambda_{JJ}
\end{array}\right]\left(m-\mu\right) +\\
&& \left[\begin{array}{ccccc}
0 & \Lambda_{12} & \cdots & \Lambda_{1\left(J-1\right)} & \Lambda_{1J}\\
\vdots &  & \ddots &  & \vdots\\
\Lambda_{J1} & \Lambda_{J2} & \cdots & \Lambda_{J\left(J-1\right)} & 0
\end{array}\right]\left(m-\mu\right)\Leftrightarrow\\
0 & = & \Lambda \left(m-\mu\right) \Leftrightarrow\\
m & = & \mu.
\end{eqnarray*}
The last step follows from the assumption that $\Sigma$ (and hence $\Lambda$)
is invertible.  It follows that $\mu$ is the unique stable point of
\eq{mvn_stable_point}.

\end{proof}

\begin{proposition} \label{prop:lrvb_mvn}
  Assume we are in the setting of \lem{lrvb_mvn}, where additionally
  $\mu$ and $\Sigma$ are on the interior of the feasible parameter space.
  Then the LRVB covariance estimate exactly captures the true covariance,
  $\hat{\Sigma} = \Sigma$.


\end{proposition}

\begin{proof}

  Consider the perturbation for LRVB defined in \eq{perturbed_dens}.
  By perturbing the log likelihood, we change both the true means $\mu_t$
  and the variational solutions, $m_t$. The result is a valid
  density function since the original $\mu$ and $\Sigma$ are on the
  interior of the parameter space.
  By \lem{lrvb_mvn}, the MFVB solutions are exactly the true
  means, so $m_{t,j} = \mu_{t,j}$, and the derivatives are the same
  as well.  This means that the first term in \eq{lrvb_est} is
  not approximate, i.e.
  \begin{equation*}
  \frac{d \mpq_{t}}{d t^{T}}
    = \frac{d}{d t^{T}} \mbe_{p_{t}} \theta
    = \Sigma_{p_{t}},
  \end{equation*}
  It follows from the arguments in \app{SEM} that the LRVB covariance
  matrix is exact, and $\hat{\Sigma} = \Sigma$.
  
\end{proof}

One final result about the multivariate normal will simplify
some of the leverage score calculations to follow.  
The variational distribution in \eq{mvn_variational_dist}
has both linear and quadratic sufficient statistics for
$\theta_{j}$.  That is, the full set of variational parameters
are $\tilde\theta_{j} := \left(\theta_{j}^T, s_{j}^T\right)^T$, where
$s_{j} = \textrm{Vec}\left( \theta_j \otimes \theta_j\right)$.
(The $\textrm{Vec}$ operator stacks a matrix columnwise into a vector,
and $\otimes$ denotes the Kronecker product.)

Stictly speaking, \eq{lrvb_est} requires calculating derivatives
for all the sufficient statistics, not just the statistics
that we are interested in.  However, when calculating
the LRVB covariance for the mean of a normal distribution, we
can effectively ignore the $s$ terms and apply \eq{lrvb_est}
only to the $\theta$ terms.

\begin{lemma}\label{lem:mvn_second_moments}
The LRVB covariance matrix for the mean of a multivariate
normal distribution does not depend on the sensitivities
to the quadratic sufficient statistics.  Specifically,
\begin{equation}\label{eq:mvn_second_moments}
\hat\Sigma_\theta =
    \left(I - \frac{\partial M_\theta}{\partial m_\theta}\right)^{-1}
    \Sigma_{q^*,\theta}
\end{equation}
\end{lemma}

\begin{proof}

We will evaluate the terms in \eq{lrvb_est} for the full
parameter vector $\tilde\theta$, and show that the
submatrix of $\hat\Sigma$ corresponding to $\theta$ is
given by \eq{mvn_second_moments}.

Partition the matrices $H$ and $\Sigma_{q^*}$ from
\eq{lrvb_est} into blocks for
$\theta$ and $s$.  We will use $V_\theta$ and $V_s$ to denote
the variational variance of $\theta$ and $s$, and $V_{\theta s}$
to denote the variational covariance between $\theta$ and $s$.
All the $V$ terms are given by standard properties of the multivariate
normal distribution.
\begin{eqnarray*}
H & = & \left(\begin{array}{cc}
\frac{\partial\eta_{\theta}}{\partial m_\theta} &
    \frac{\partial\eta_{\theta}}{\partial m_s}\\
\frac{\partial\eta_{s}}{\partial m_\theta} & 0
\end{array}\right)=\left(\begin{array}{cc}
\frac{\partial\eta_{\theta}}{\partial m_\theta} & 0\\
0 & 0
\end{array}\right)\\
\Sigma_{q^{*}} & = & \left(\begin{array}{cc}
V_{\theta} & V_{\theta s}\\
V_{\theta s}^{T} & V_{s}
\end{array}\right)
\end{eqnarray*}

In the formula for $H$, we have used the observation from
\eq{mvn_variational_dist} that the terms of $s$
never co-occur with any other terms of $\tilde\theta$, so that
$\frac{\partial \eta_s}{\partial m_\theta^T} = \frac{\partial \eta_\theta^T}{\partial m_s} = 0$.

First, we calculate:

\begin{eqnarray*}
\left(I-\Sigma_{q^{*}}H\right) & = & I-\left(\begin{array}{cc}
V_{\theta} & V_{\theta s}\\
V_{\theta s}^{T} & V_{s}
\end{array}\right)\left(\begin{array}{cc}
\frac{\partial\eta_{\theta}}{\partial m_\theta} & 0\\
0 & 0
\end{array}\right)\\
 & = & \left(\begin{array}{cc}
I-V_{\theta}\frac{\partial\eta_{\theta}}{\partial m_\theta} & 0\\
-V_{\theta s}^{T}\frac{\partial\eta_{\theta}}{\partial m_\theta} & I
\end{array}\right)
\end{eqnarray*}

Using the Schur inverse and the fact that the upper right hand
corner is the $0$ matrix, we can write

\begin{eqnarray*}
\left(I-\Sigma_{q^{*}}H\right)^{-1} & = & \left(\begin{array}{cc}
\left(I-V_{\theta}\frac{\partial\eta_{\theta}}{\partial m_\theta}\right)^{-1} & 0\\
Q_{\theta s} & Q_{ss}
\end{array}\right),
\end{eqnarray*}
where $Q_{\theta s}$ and $Q_{ss}$ are simply placeholders for the rest of the
inverse.  Multiplying by $\Sigma_{q^*}$ gives that the $\theta$-sized
upper-left corner of $\hat\Sigma$ is
\begin{eqnarray*}
\hat{\Sigma}_{\theta\theta} & = & \left(I-V_{\theta}\frac{\partial\eta_{\theta}}{\partial m_\theta}\right)^{-1} V_\theta
\end{eqnarray*}

This is the same as \eq{mvn_second_moments} and identical to what
we would have gotten by applying \eq{lrvb_est} to $\theta$ alone,
ignoring the $s$ dependence.

\end{proof}

\subsection{Comparison with supplemented expectation-maximization}\label{subsec:SEM}

This result about the multivariate normal distribution
draws a connection between LRVB
corrections and the ``supplemented expectation-maximization'' (SEM)
method of \cite{meng:1991:using}.  SEM is an asymptotically
exact covariance correction for the EM algorithm that transforms
the full-data Fisher information matrix into the observed-data Fisher
information matrix using a correction that is formally similar to
\eq{lrvb_est}.  In this section, we argue that this similarity
is not a coincidence; in fact the SEM correction is an
asymptotic version of LRVB with two variational blocks,
one for the missing data and one for the unknown parameters.

Although LRVB as described here requires a prior 
(unlike SEM, which supplements the MLE),
the two covariance corrections coincide when
the full information likelihood is approximately log quadratic
and proportional to the posterior, $p(\theta \vert x)$.
This might be expected to occur when we have a large number
of independent data points informing each parameter---i.e.,
when a central limit theorem applies and the priors do not
affect the posterior.
In the full information likelihood, some
terms may be viewed as missing data, whereas in the Bayesian
model the same terms may be viewed as latent parameters,
but this does not prevent us from formally comparing the two methods.

We can draw a term-by-term analogy with
the equations in \cite{meng:1991:using}. We denote variables
from the SEM paper with a superscript ``$SEM$'' to avoid confusion.
MFVB does not differentiate between missing
data and parameters to be estimated, so our $\theta$ corresponds to
$(\theta^{SEM}, Y_{mis}^{SEM})$ in \cite{meng:1991:using}.
SEM is an asymptotic
theory, so we may assume that $(\theta^{SEM}, Y_{mis}^{SEM})$ have a
multivariate normal
distribution, and that we are interested in the mean and covariance of
$\theta^{SEM}$.

In the E-step of \cite{meng:1991:using}, we replace $Y_{mis}^{SEM}$ with
its conditional expectation given the data and other $\theta^{SEM}$.
This corresponds precisely to \eq{mvn_stable_point}, taking
$\theta_j = Y_{mis}^{SEM}$.  In the M-step, we find the maximum
of the log likelihood with respect to $\theta^{SEM}$, keeping
$Y_{mis}^{SEM}$ fixed at its expectation.  Since the mode
of a multivariate normal distribution is also its mean,
this, too, corresponds to \eq{mvn_stable_point}, now taking
$\theta_j = \theta^{SEM}$.

It follows that the MFVB and EM fixed point equations are the same;
i.e., our $M$ is the same as their $M^{SEM}$, and
our $\partial M / \partial m$ of \eq{dm_dt} corresponds
to the transpose of their $DM^{SEM}$, defined in \eqw{2.2.1}
of \cite{meng:1991:using}.  Since the ``complete information'' corresponds to
the variance of $\theta^{SEM}$ with fixed values for $Y_{OBS}^{SEM}$,
this is the same as our $\Sigma_{q^*,11}$, the variational covariance,
whose inverse is $I_{oc}^{-1}$.  Taken all together, this means that
equation (2.4.6) of \cite{meng:1991:using} can be
re-written as our \eq{lrvb_est}.
\begin{align*}
V^{SEM} =& I_{oc}^{-1} \left(I - DM^{SEM}\right)^{-1} \Rightarrow\\
\Sigma =& \Sigma_{q^*} \left(I - \left(\frac{\partial M}{\partial m^T}\right)^T \right)^{-1}
       = \left(I - \frac{\partial M}{\partial m^T} \right)^{-1} \Sigma_{q^*}
\end{align*}

\section{Leverage scores} \label{app:leverage}

In a linear model $y_i = \beta^T x_i + \epsilon$, leverage score estimates how much
influence each observation $x_i$ has on its fitted value, $\hat{y_i} = \hat\beta x_i$,
through its influence on $\hat\beta$.  In an analogous Bayesian way, we can
use LRVB to estimate the correlation between infinitesimal noise in our observed
data and our posterior estimates of $\theta = \left(\mu, \sigma, p\right)$ in
the model of \mysec{experiments} .

In this appendix, we first show that covariance-based ``leverage scores'' described
in \mysec{normal_sensitivity} are the same as classical leverage scores for
linear models.  Then, we derive the leverage scores for the means of a
normal mixture model.

\subsection{Linear model leverage scores}

Let us define a classical linear regression with known variance as
\begin{eqnarray*}
y_{i} & \sim & \gauss \left(\beta^{T}x_{i}, \sigma^{2}\right)\\
\log p\left(Y | \beta \right) & = & 
   -\frac{1}{2\sigma^{2}}\beta^{T}X^{T}X\beta+\frac{1}{\sigma^{2}}Y^{T}X\beta + \mathrm{constant}.
\end{eqnarray*}

Here, in order to take advantage of familiar matrix formulas for linear regression,
we will use capital letters to denote vectors and matrices in this section.
That is, $Y$ is the vector of scalars $y_{i}$, $X$ is the matrix
formed by stacking the observations $x_{i}^{T}$. To recover leverage
scores, suppose that instead of $y_i$, we observe normal random variables $y_i^{*}$, where:
\begin{eqnarray*}
E(y_{i}^{*} \vert y_{i}) & = & y_{i}\\
Var(y_{i}^{*} \vert y_{i}) & = & \epsilon.
\end{eqnarray*}

The variables $y_i$ and $y_i^*$ are analogous to the variables $x_i$ and $x_i^{*}$
of \mysec{normal_sensitivity}, respectively.
We will then use MFVB to fit this model where the parameters to be estimated
are $\theta = (Y^{T}, \beta^{T})^{T}$
and we have a uniform improper prior on $\beta$.
Since the posterior is multivariate
normal, in this case the LRVB covariance matrices for $\theta$ will be
exact in light of \app{SEM}.

The sufficient statistics for $Y$ include quadratic terms, $y_i^2$,
that are correlated with the linear sufficient statistics.  Ordinarily,
one must also include derivatives with respect to these quadratic
sufficient statistics when applying \eq{lrvb_est} (as is done in
\mysec{mixture_leverage}).  However, since the posterior is
multivariate normal, we can apply \lem{mvn_second_moments} and only
consider the sensitivity to $Y$.

The terms in \eq{lrvb_est} are given by:
\begin{eqnarray*}
\Sigma_{q^*_Y} & = & \var(Y|\beta, Y^{*}) = \epsilon I_{Y}\\
\Sigma_{q^*_\beta} & = & \var(\beta|Y, Y^{*})  = \left(X^{T}X\right)^{-1}\sigma^{2}\\
\frac{\partial \eta_{\beta}}{\partial m_Y^T} & = & \frac{1}{\sigma^{2}}X^{T}\\
\frac{\partial \eta_{Y}}{\partial m_\beta^T} & = & \frac{1}{\sigma^{2}}X\Rightarrow\\
I- \Sigma_{q^*} H & = & \left(\begin{array}{cc}
I_{\beta} & -\sigma^{2}\left(X^{T}X\right)^{-1}X^{T}\\
    -\frac{\epsilon}{\sigma^{2}}X & I_{Y}
\end{array}\right).
\end{eqnarray*}

The upper-left ($\beta$) component of $(I - \Sigma_{q^*} H )^{-1}$ can be calculated
with the Schur complement:
\begin{eqnarray*}
\left(I - \Sigma_{q^*} H \right)_{\beta\beta}^{-1} & = &
    \left(I_{\beta}-\frac{\epsilon}{\sigma^{2}}\left(X^{T}X\right)^{-1}X^{T}X\right)^{-1}\\
 & = & \left(1-\frac{\epsilon}{\sigma^{2}}\right)^{-1}I_{\beta}\\
 & \equiv & \alpha I_{\beta},
\end{eqnarray*}
where we have defined $\alpha=\sigma^{2}\left(\sigma^{2}-\epsilon\right)^{-1}$.
Note that $\lim_{\epsilon\rightarrow0}\alpha=1$. This gives the rest
of the inverse and the covariance between $Y$ and $\beta$:
\begin{eqnarray*}
\left(I - \Sigma_{q^*} H \right)^{-1} & = & \left(\begin{array}{cc}
\alpha I_{\beta} & \alpha\left(X^{T}X\right)^{-1}X^{T}\\
\alpha \frac{\epsilon}{\sigma^{2}}X & \left(I_{Y}-\frac{\epsilon}{\sigma^{2}}P_{X}\right)^{-1}
\end{array}\right)\\
\Sigma_p & = & \left(\begin{array}{cc}
  \alpha I_{\beta} & \alpha\left(X^{T}X\right)^{-1}X^{T}\\
  \alpha \frac{\epsilon}{\sigma^{2}}X & \left(I_{Y}-\frac{\epsilon}{\sigma^{2}}P_{X}\right)^{-1}
  \end{array}\right)\left(\begin{array}{cc}
  \sigma^{2}\left(X^{T}X\right)^{-1} & 0\\
  0 & \epsilon I_{Y}
  \end{array}\right)\\
 & = & \left(\begin{array}{cc}
    \alpha\sigma^{2}\left(X^{T}X\right)^{-1} &
    \epsilon\alpha\left(X^{T}X\right)^{-1}X^{T}\\
    \epsilon\alpha X\left(X^{T}X\right)^{-1} &
    \epsilon\left(I_{Y}-\frac{\epsilon}{\sigma^{2}}P_{X}\right)^{-1}
\end{array}\right),
\end{eqnarray*}
where $P_{X}=X\left(X^{T}X\right)^{-1}X^{T}$ is the projection matrix
onto $X$. This says that
\begin{eqnarray*}
\cov(\beta,Y \vert Y^{*}) & = & \epsilon\alpha\left(X^{T}X\right)^{-1}X^{T}\Rightarrow\\
\cov(\hat{Y},Y\vert Y^{*}) & = & Cov(X\beta,Y)=XCov(\beta,Y)\\
 & = & \epsilon\alpha P_{X}.
\end{eqnarray*}

Since $\epsilon\rightarrow0\Rightarrow\alpha\rightarrow1$, the covariance
between $\hat{y}_{i}$ and $z_{i}$ is proportional to the diagonal
of $P_{X}$, which is exactly the classical leverage score.

\subsection{Normal mixture leverage scores} \label{sec:mixture_leverage}

We now consider leverage scores in the setting of \mysec{normal_sensitivity},
The new posterior with perturbed $x$ observations is the original
posterior plus a term for $x^*$:
\begin{eqnarray*}
\log p\left(\mu, \sigma, a, z, x | x^* \right) & = &
    \log p\left(\mu, \sigma, a, z | x \right) +  \log p\left(x | x^* \right) + \mathrm{constant} \\
& = & \log p\left(\mu, \sigma, a, z | x \right) - 
      \frac{1}{2}\sigma_{x}^{-2}\sum_{i}\left(x_{i}-x_{i}^{*}\right)^{2} + \mathrm{constant}.
\end{eqnarray*}

We can imagine estimating each of the unobserved $x_i$ its own varational distribution
with sufficient statistics $x_i$ and $x_i^2$, though since
we are adding infinitesimal noise, it is not necessary to actually re-fit the model.
Infinitesimal noise in $x$ will not change the point estimates of $\theta$, and
since $\sigma_{x} \approx 0$, only the $x_i^{*}$ terms matter for the variational
posterior of $x_i$.
Using standard properties of the normal distribution and the fact that
$\sigma_{x} \approx 0$, the variational
expectations of the sufficient statistics are then given by:
\begin{eqnarray}\label{eq:normal_sens_vb_expectations}
E_{q*}\left(x_i\right) & = & x_i^{*}\\
E_{q*}\left(x_i^2\right) & = & \sigma_{x}^2 + x_i^{*2}\\
\var_{q*}\left(x_i\right) & = & \sigma_{x}^{2}\\
\var_{q*}\left(x_i^2\right) & = &
    4x_{i}^{*2}\sigma_{x}^{2} + 2 \sigma_{x}^{4}
    \approx 4x_{i}^{*2} \sigma_{x}^{2}\\
\cov_{q*}\left(x_i^2\right) & = & x_i^{*} \sigma_{x}^{2}.
\end{eqnarray}

It will be notationally convenient to stack the sufficient
statistics $x_i$ and
$x_i^2$ in a single vector, simply called $x$.  $E_{q*}\left(x\right)$
and $\var_{q*}\left(x\right)$, the variational mean and
covarivance of $x$, can be read off \eq{normal_sens_vb_expectations}.
We will also define $V_x$ by:
\begin{eqnarray*}
\var_{q*}\left(x\right) & := & \sigma_{x}^{2} V_x .
\end{eqnarray*}

To get the LRVB covariance, we need only to calculate the quantities
in equation \ref{eq:lrvb_est}.  In particular, we are interested in the
sub-matrix $\hat{\Sigma}_{\theta x}$, the estimated covariance
between $\theta$ and $x$.
Although we will derive the covariance between $\theta$ and all of
$x$, we can keep in mind that the leverage scores are actually
the submatrix of this covariance that corresponds to the $x_i$ terms,
not the $x_i^2$ terms.

To aid our computation, we will
use Schur compliments and the fact that $\sigma_{x}^2 \approx 0$.
In order to make the notation tidier, we will use some shorthand notation
relative to the main body of the text:
\begin{eqnarray*}
V & :=  &\Sigma_{q*}\\
R & :=  &\Sigma_{q*} H.
\end{eqnarray*}

We partition each matrix into $\theta$, $x$, and $z$ blocks:
\begin{eqnarray*}
\hat\Sigma & = & \left(\begin{array}{ccc}
\hat\Sigma_{\theta\theta} & \hat\Sigma_{\theta X} & \hat\Sigma_{\theta Z}\\
\hat\Sigma_{X\theta} & \hat\Sigma_{XX} & \hat\Sigma_{XZ}\\
\hat\Sigma_{Z\theta} & \hat\Sigma_{ZX} & \hat\Sigma_{ZZ}
\end{array}\right)\\
R & = & \left(\begin{array}{ccc}
R_{\theta} & R_{\theta X} & R_{\theta Z}\\
R_{X\theta} & 0 & R_{XZ}\\
R_{Z\theta} & R_{ZX} & 0
\end{array}\right)\\
V & = & \left(\begin{array}{ccc}
V_{\theta} & 0 & 0\\
0 & \sigma_{x}^{2} V_x & 0\\
0 & 0 & V_{Z}
\end{array}\right).
\end{eqnarray*}

We are interested in using equation \ref{eq:lrvb_est}, i.e.
$\Sigma = \left(I-R\right)^{-1} V$, to find the sub-matrix in
$\Sigma_{x\theta}$.  (We could just as well find $\Sigma_{\theta x}$.)
First, note that one can eliminate $z$ immediately with a Schur
complement.  In general, if the matrices partition into two groups $A$
and $B$, then
\begin{eqnarray}\label{eq:z_elimination}
\hat\Sigma_{A} & = & \left[ I - R_{AA} - R_{AB} \left(I - R_{BB}\right)^{-1} R_{BA} \right]^{-1} V_{A}.
\end{eqnarray}

In this case, let $B$ refer to the $z$ variables and $a$ to everything else.  Noting that $R_{ZZ} = 0$
and applying formula \ref{eq:z_elimination} gives
\begin{eqnarray*}
\left(\begin{array}{cc}
\hat\Sigma_{\theta\theta} & \hat\Sigma_{\theta X}\\
\hat\Sigma_{X\theta} & \hat\Sigma_{XX}
\end{array}\right) & = & \left[\left(\begin{array}{cc}
I_{\theta} & 0\\
0 & I_{X}
\end{array}\right)-\left(\begin{array}{cc}
R_{\theta} & R_{\theta X}\\
R_{X\theta} & 0
\end{array}\right)-\left(\begin{array}{c}
R_{\theta Z}\\
R_{XZ}
\end{array}\right)\left(\begin{array}{cc}
R_{Z\theta} & R_{ZX}\end{array}\right)\right]^{-1}\left(\begin{array}{cc}
V_{\theta} & 0\\
0 & \sigma_{x}^{2}V_x
\end{array}\right)\\
 & = & \left[\left(\begin{array}{cc}
I_{\theta}-R_{\theta} & -R_{\theta X}\\
-R_{X\theta} & I_{X}
\end{array}\right)-\left(\begin{array}{cc}
R_{\theta Z}R_{Z\theta} & R_{\theta Z}R_{ZX}\\
R_{XZ}R_{Z\theta} & R_{XZ}R_{ZX}
\end{array}\right)\right]^{-1}\left(\begin{array}{cc}
V_{\theta} & 0\\
0 & \sigma_{x}^{2}V_x
\end{array}\right)\\
 & = & \left(\begin{array}{cc}
\left(I_{\theta}-R_{\theta}-R_{\theta Z}R_{Z\theta}\right) & \left(-R_{\theta X}-R_{\theta Z}R_{ZX}\right)\\
\left(-R_{X\theta}-R_{XZ}R_{Z\theta}\right) & \left(I_{X}-R_{XZ}R_{ZX}\right)
\end{array}\right)^{-1}\left(\begin{array}{cc}
V_{\theta} & 0\\
0 & \sigma_{x}^{2}V_x
\end{array}\right).
\end{eqnarray*}

It will be enough to get the first row, $\left(\hat\Sigma_{\theta\theta}, \hat\Sigma_{\theta X}\right)$,
and for that we can use the Schur inverse.
\begin{eqnarray*}
BD^{-1}C & = & \left(R_{\theta X}+R_{\theta Z}R_{ZX}\right)\left(I_{X}-R_{XZ}R_{ZX}\right)^{-1}\left(R_{X\theta}+R_{XZ}R_{Z\theta}\right)\\
A-BD^{-1}C & = & I_{\theta}-R_{\theta}-R_{\theta Z}R_{Z\theta}-\left(R_{\theta X}+R_{\theta Z}R_{ZX}\right)\left(I_{X}-R_{XZ}R_{ZX}\right)^{-1}\left(R_{X\theta}+R_{XZ}R_{Z\theta}\right)\\
BD^{-1} & = & -\left(R_{\theta X}+R_{\theta Z}R_{ZX}\right)\left(I_{X}-R_{XZ}R_{ZX}\right)^{-1}.
\end{eqnarray*}

Given these quantities, since the $V$ matrix is block diagonal,
\begin{eqnarray*}
\hat\Sigma_{\theta X} & = & -\left(A-BD^{-1}C\right)^{-1}BD^{-1} V_x \sigma_{x}^{2}.
\end{eqnarray*}
It will be helpful to simplify this by taking $\sigma_{x}^{2}\rightarrow0$.
To aid in this, write:
\begin{eqnarray*}
R_{XZ} & = & \sigma_{x}^{2}Q_{XZ}\\
R_{X\theta} & = & \sigma_{x}^{2}Q_{X\theta}\\
D^{-1}=\left(I_{X}-R_{XZ}R_{ZX}\right)^{-1} & = & \left(I_{X}-\sigma_{x}^{2}Q_{XZ}R_{ZX}\right)^{-1}\\
 & \approx & I_{X}+\sigma_{x}^{2}Q_{XZ}R_{ZX}.
\end{eqnarray*}
Then:
\begin{eqnarray*}
A-BD^{-1}C & \approx & I_{\theta}-R_{\theta}-R_{\theta Z}R_{Z\theta}-\sigma_{x}^{2}\left(R_{\theta X}+R_{\theta Z}R_{ZX}\right)\left(I_{X}+\sigma_{x}^{2}Q_{XZ}R_{ZX}\right)\left(Q_{X\theta}+Q_{XZ}R_{Z\theta}\right)\\
 & \approx & I_{\theta}-R_{\theta}-R_{\theta Z}R_{Z\theta}-\sigma_{x}^{2}\left(R_{\theta X}+R_{\theta Z}R_{ZX}\right)\left(Q_{X\theta}+Q_{XZ}R_{Z\theta}\right)\\
\left(A-BD^{-1}C\right)^{-1} & \approx & \left(I_{\theta}-R_{\theta}-R_{\theta Z}R_{Z\theta}-\sigma_{x}^{2}\left(R_{\theta X}+R_{\theta Z}R_{ZX}\right)\left(Q_{X\theta}+Q_{XZ}R_{Z\theta}\right)\right)^{-1}\\
 & \approx & \left(I_{\theta}-R_{\theta}-R_{\theta Z}R_{Z\theta}\right)^{-1}  \times\\
 & & \left(I_{\theta}+\sigma_{x}^{2}\left(I_{\theta}-R_{\theta}-R_{\theta Z}R_{Z\theta}\right)^{-1} \left(R_{\theta X}+R_{\theta Z}R_{ZX}\right)\left(Q_{X\theta}+Q_{XZ}R_{Z\theta}\right)\right).
\end{eqnarray*}
Similarly,
\begin{eqnarray*}
BD^{-1} & = & -\left(R_{\theta X}+R_{\theta Z}R_{ZX}\right)\left(I_{X}-\sigma_{x}^{2}Q_{XZ}R_{ZX}\right)^{-1}\\
 & \approx & -\left(R_{\theta X}+R_{\theta Z}R_{ZX}\right)\left(I_{X}+\sigma_{x}^{2}Q_{XZ}R_{ZX}\right).
\end{eqnarray*}
This uses the matrix version of this Taylor expansion:
\begin{eqnarray*}
\frac{1}{1-r} & \approx & 1+r\\
\frac{1}{x-r} & = & \frac{x^{-1}}{1-x^{-1}r}\approx x^{-1}\left(1+x^{-1}r\right).
\end{eqnarray*}
as well as eliminating any term that exhibits terms that have second
or higher powers of $\sigma_{x}^{2}$. Observe that if $\sigma_{x}^{2}=0$,
then this gives:
\begin{eqnarray*}
\hat\Sigma_{\theta\theta}^{0}V_{\theta}^{-1} & = & \left(I_{\theta}-R_{\theta}-R_{\theta Z}R_{Z\theta}\right)^{-1}.
\end{eqnarray*}

This is the covariance of $\theta$ before performing the sensitivity
analysis. Substitute this in:
\begin{eqnarray*}
\left(A-BD^{-1}C\right)^{-1} & \approx & \hat\Sigma_{\theta\theta}^{0}V_{\theta}^{-1}\left(I_{\theta}+\sigma_{x}^{2}\hat\Sigma_{\theta\theta}^{0}V_{\theta}^{-1}\left(R_{\theta X}+R_{\theta Z}R_{ZX}\right)\left(Q_{X\theta}+Q_{XZ}R_{Z\theta}\right)\right).
\end{eqnarray*}

Now things are tidy enough to plug in for $\hat\Sigma_{\theta X}$.
\begin{eqnarray*}
\hat\Sigma_{\theta X} & = & -\left(A-BD^{-1}C\right)^{-1}BD^{-1}\left(\sigma_{x}^{2} V_{x}\right)\\
 & \approx & \hat\Sigma_{\theta\theta}^{0}V_{\theta}^{-1}\left(I_{\theta}+\sigma_{x}^{2}\hat\Sigma_{\theta\theta}^{0}V_{\theta}^{-1}\left(R_{\theta X}+R_{\theta Z}R_{ZX}\right)\left(Q_{X\theta}+Q_{XZ}R_{Z\theta}\right)\right)\times\\
 &  & \left(R_{\theta X}+R_{\theta Z}R_{ZX}\right)\left(I_{X}+\sigma_{x}^{2}Q_{XZ}R_{ZX}\right)\sigma_{x}^{2}V_{x}\\
 & \approx & \sigma_{x}^{2}\hat\Sigma_{\theta\theta}^{0}V_{\theta}^{-1}\left(R_{\theta X}+R_{\theta Z}R_{ZX}\right)\left(I_{X}+\sigma_{x}^{2}Q_{XZ}R_{ZX}\right) V_{x}\\
 & \approx & \sigma_{x}^{2}\hat\Sigma_{\theta\theta}^{0}V_{\theta}^{-1}\left(R_{\theta X}+R_{\theta Z}R_{ZX}\right) V_{x}.
\end{eqnarray*}

The final result is appealingly simple.
\begin{eqnarray}
  \nonumber
\hat\Sigma_{\theta x} & = & \sigma_{x}^{2} L_{\theta X}\\
  \label{eq:lrvb_lev_scores}
L_{\theta x} & := & \hat\Sigma_{\theta} \hat\Sigma_{q*,\theta}^{-1}
                    \left(R_{\theta X} + R_{\theta Z} R_{ZX}\right) V_{x}.
\end{eqnarray}
The quantities $L_{\theta x}$ are the leverage scores that are
plotted in \fig{LeverageGraph}.


\end{document}